\documentclass[12pt]{article}
\usepackage{geometry}
\usepackage{float}
\usepackage{amsmath}
\usepackage[inline]{enumitem}
\usepackage{cmll}
\usepackage{color}
\usepackage{amssymb}
\usepackage{mathtools}
\usepackage[T4,T1]{fontenc}
\usepackage{booktabs}
\usepackage{graphicx}
\usepackage{float}
\usepackage[linesnumbered,vlined,commentsnumbered,ruled]{algorithm2e}
 \usepackage{amsfonts}
 \newtheorem{theorem}{Theorem}
  \newtheorem{definition}{Definition}
   \newtheorem{proof}{Proof}
    
\usepackage{authblk}
\usepackage{hyperref}
\begin{document}

\title{\rule{17cm}{2.4pt} \\ \textbf{Finite Mixture Model of Nonparametric Density Estimation using Sampling Importance Resampling for Persistence Landscape}\\\rule{17cm}{1pt}}

\author{
    Farzad Eskandari \\
    \small{Faculty of Mathematics and Computer Science}\\
\small{Allameh Tabatabai University}\\
    \small{Tehran, Iran} \\
    \small{ffeskandari@yahoo.com}
    \and
    Soroush Pakniat\\
    \small{Faculty of Mathematics and Computer Science}\\
\small{Allameh Tabatabai University}\\
\small{Tehran, Iran} \\
\small{sorush.pakniat@atu.ac.ir}
}
\date{}
\maketitle
\begin{abstract}
Considering the creation of persistence landscape on a parametrized curve and structure of sampling, there exists a random process for which a finite mixture model of persistence landscape (FMMPL) can provide a better description for a given dataset. In this paper, a nonparametric approach for computing integrated mean of square error (IMSE) in persistence landscape has been presented. As a result, FMMPL is more accurate than the another way. Also, the sampling importance resampling (SIR) has been presented a better description of important landmark from parametrized curve. The result, provides more accuracy and less space complexity than the landmarks selected with simple sampling.
\end{abstract}
\providecommand{\keywords}[1]{\textbf{\textit{Keywords:}} #1}
\section{Introduction}
Recently, a considerable growth has been witnessed in the rate of data generation in modern sciences and engineering. Dimensionality is one of the restrictions in visualization and analysis of the dataset. The geometry and topology can be considered as effective tools for studying distance function and how connecting the component and classification of loops. A few of the important points of topology are less sensitive to the selected metric compared to the geometric approach. Typically, the selected metric cannot represent the real distance between points, be coordinated freely, preserve properties of space under continuous deformation, and represent shape as compressed\cite{CarlssonTopologyData}, \cite{CarlssonSurveyActa}. Topological data analysis(TDA) has two fundamental issues \cite{Ghrist}: 
\begin {enumerate*}[label=(\roman*)]
\item how one infers high dimensional structure from low dimensional representation; \item how one assembles discrete points into global structure. 
\end {enumerate*} The mapper and persistence homology are two major approaches in TDA. The mapper method covers a metric space and inverse of mapping, construct the simplicial complex with control of resolution, and represent statistical inference to shape graph \cite{singh}. The persistence homology constructs simplicial complex from point cloud depending on proximity parameter, which is a topological object. Next, based on a theorem in algebraic topology, it applies persistence homology to this object. \newline
Although nodes are connected in all graph-based models, in many real systems such a nondyadic relation, they cannot represent the nature of the system. One possible solution for this restriction is the use of hypergraph that records all relations; however, it is not interested due to its computational complexity. In this regard, we use simplicial complex that can represent compact, computable, and non-pairwise subgroup relations of interest from the population. \newline
The use of barcode and persistence diagram combined with statistics and machine learning provides an alternative approach that called persistence landscape. It is a function that can use the vector space structure of its underlying function space. This function space is a separable Banach space that allows creating a random variable with a value in such space. \newline
The space of the persistence diagram is geometrically very complicated. In order to estimate Fr$\acute{e}$chet mean from the set of diagrams ($X_1 , \ldots , X_n$) by \cite{TurnerFrechet}, the authors in \cite{TurnerMeans} showed that the mean of the diagram is not unique but is unique for a special class of persistence diagram. Moreover, it is known that the space of persistence diagram is analogous to $L^p$ space. As a result, it is not plausible to use any parametric models for distribution. In this regard,  \cite{RobinsonArXiv} used randomization test where two sets of diagrams were drawn from the same single distribution of diagrams. In \cite{MileykoProbabilityMeasures}, a theoretical basis is provided for a statistical treatment that supports expectations, variance, percentiles, and conditional probabilities on persistence diagrams. In \cite{MichelStatistical}, an alternative function is introduced on the statistical analysis of the distance to measure (DTM) that estimates the persistence diagram on metric space. In \cite{BlumbergStatistical}, persistence homology is adapted for computing confidence interval and hypothesis testing. Finally, in \cite{ChazalStochastic}, the convergence of the average landscapes and bootstrap is investigated. \newline
The scalar function can be used instead of metric space because of its ease of use. It is of note that since a small perturbation of landmarks has small changes at the scalar function, we need to prove the stability for alternative function. A proof of stability for persistence landscape and for persistence entropy is presented in \cite{JMLR:v16:bubenik15a} and \cite{Atienta}, respectively. \newline
There are several interesting applications of persistence invariant; i.e., the use of persistence homology for solving coverage problem in sensor networks by \cite{SilvaSensorNetworks}, modeling the spaces to patches pixels, and describing the global topological structure for patches \cite{CarlssonNaturalImage}, computing persistence homology for identifying the global structure of similarities between data by \cite{WagnerTowards}, applying persistence measures for the analysis of the observed spatial distribution of galaxies with Megaparsec scales by \cite{PranavCosmicWeb}, simultaneous use of barcode and persistence landscape for identifying the changes in community structure between brain regions that form loops in functional brain for three days \cite{Stolz2017}. Also, in \cite{entropy1} and \cite{entropy2}, based on Shannon entropy, a persistent entropy was defined on normalized barcodes. \newline
TDA has some fundamental aspects. In \cite{BubenikCategorification},  a persistence homology was recreated based on a category theory and some features of $(\mathbb{R}, \leq)$, consisting of a set of objects and morphisms, were investigated. Also, \cite{BubenikGromovHausdorff} presented a generalization of Hausdorff distance, Gromov-Hausdorff distance, and the space of metric spaces in the form of a categorical view. To generalize the persistence module with the category theory and soft stability theorem see \cite{BubenikMetricGeneralized}. Moreover, in \cite{BubenikHigherInterpolation}, the authors present a categorical language for construction embedding of a metric space into the metric space of persistence module. \newline
The shape is all the geometrical information that remains when location, scale, and rotational effects are removed from a given object \cite{ShapeAnalysis}. So, a shape can be represented by a finite set of the points located on specific shape regions, which are called landmarks or sampling points. There are three basic types of landmarks in our applications; i.e., scientific, mathematical, and pseudo-landmarks. \newline
In the present work, we aimed at proposing a nonparametric inference of data to infer an unknown quantity to keep the number of underlying assumptions as low as possible. Our approach would be of great assistance in case the modeler is unable to find a theoretical distribution that provides a good model for the input data. Populations of individuals often are divided into subgroups. The task in examining a sample of measurements to discern and describe subgroups of individuals -- even when there is no observable variable that readily indexes into the subgroup an individual properly belongs -- is sometimes called as unsupervised clustering in the literature. Indeed, mixture models may be generally considered as a subset of clustering methods known as model-based clustering \cite{Chauveau}. Identifiability is a major concern of finite mixture models. Estimation procedures may not be well defined and asymptotic theory may not be held if a model is not identifiable. The main objective of this work is to present a nonparametric density estimation for mixing the persistence landscape in Banach space. Also, the elastic shape analysis search is performed for an optimal reparameterization, which relies on the square root velocity function. This function is invariance to translation. We apply sampling importance resampling for posterior distribution that estimated characteristics of the whole population of the curve and gathering all the needed information, which is a time-consuming and costly task. Moreover, each methods that can extract landmarks from curve accurately could provide a better estimation for the parameterised curve. Finally, we make inferences about the population with the help of sampling.
\newline
The remainder of this paper is organized as follows: In section \ref{section:Background}, we review the necessary background of differential geometry, sampling importance resampling, persistence landscape, finite mixture model, and nonparametric density estimation. In section \ref{section:Applications}, we apply our approach on a sampling of the object and evaluate sampling with an average of distances and a mixture model with IMSE. Finally, in the discussion section, we propose several approaches for the future studies.
\section{Background} \label{section:Background}
\subsection*{Square Root Velocity Function}
In this section, we use differentiation and integration to describe the curve in n-dimentional space \cite{baar}. Consider an interval $I \subset \mathbb{R}$, wherein a parametrized curve is a map $C:I \rightarrow \mathbb{R}^n$ that is diffrentiated infinitely. A parametrized curve $C(t)$ is regular if $\text{\.{C}}(t) \neq 0$ for all $t \in I$. For example, a circular curve around the origin $(0,0)$ with radius $r > 0$ is in the form of
\begin{align*}
\begin{array}[t]{c} C:\mathbb{R}  \rightarrow \mathbb{R}^2, \\
C(t) = 
 \begin{pmatrix}
 r.\cos(t) \\
  r.\sin(t)
 \end{pmatrix}.
\end{array}
\end{align*}
Let $C:I \rightarrow \mathbb{R}^n$ be a parametrized curve, a parameter transformation of $C$ is the bijective map $\varphi:J \rightarrow I$ such that $J \subset \mathbb{R}$ and both $\varphi$ and $\varphi^{-1}:I \rightarrow J$ are often differentiated infinitely. A paramterised curve $\tilde{C}:C \circ \varphi:J \rightarrow \mathbb{R}^n$ is called a reparametrisation of $C$. A curve is an equivalence class of regular parametrized curves such that curves can reparametrise each other. A parametrisation curve $C:\mathbb{R} \rightarrow \mathbb{R}^n$ is periodic with period $L$, if for all $t \in \mathbb{R}$, we have $C(t+L)=C(t)$ for $L > 0$. In order to compare two curves $C_1$ and $C_2$, the metric must be specified on the shape space, invariant to rigid motion, scaling, and reparametrisation. For this issue, \cite{Srivastava} introduced the square root velocity function(SRVF), which is $q(t)=\dfrac{\text{\.{C}}(t)}{\sqrt{|\text{\.{C}}(t)|}}$, where $|.|$ is the euclidean norm.
\subsection*{Sampling Importance Resampling} \label{sub:BSIR}
In order to apply Bayes rule, first, we obtain the prior distribution of landmarks which $\theta_1, \ldots, \theta_k \sim f(\theta)$. Landmark locations are certainly not independent, thus we apply order statistics on $\theta_i$ that $P(\theta_i = \theta_j)=0$ for all $i \neq j$. The distribution function of them are $\theta_{(1)} <  \ldots < \theta_{(k)} \sim g(\theta)$, $\theta_{(i)}$ is related to $\theta_i$, thus
\begin{align*}
\pi(\theta_{(1)}, \ldots, \theta_{(k)}) = g(\theta_{(1)}, \ldots, \theta_{(k)})= n \oc \prod_{i=1}^n f(\theta_{(i)})
\end{align*}
proof of this theorem in \cite{Gibbons2010}. Now, we would like to choose likelihood function for domain of parametrized curve. Let $L_\theta$ is a linear interpolation of $C(\theta)$, the elastic distance between landmarks as follow
\begin{equation}
d(C,L_\theta)=\sqrt{\int |q_C(t) - q_{L_{\theta}}|^2 dt}
\label{eq:elastic_dis}
\end{equation}
where $q_C(t)$ is $k$ times of simple random sampling of parametrized curve and $q_{L_{\theta}}$ be the linear interpolation of it.\newline
A small value of equation (\ref{eq:elastic_dis}) represents that sampling well. After that,  we compute the elastic distance, define a likelihood function as follows
\begin{align*}
f( C | \theta, \eta) \propto \eta^{-n} \exp( \dfrac{-1}{2 \eta} d^2(C,L_\theta)).
\end{align*}
The posterior can be obtained using Bayes rule as follows
\begin{align*}
\pi(\theta, \eta | C )  \propto \pi(\theta). \pi_\eta(\eta).f(C|\theta,\eta).
\end{align*}
the variance $\eta$ is a nuisance parameter, which we integrate out
\begin{equation}
\pi(\theta | C ) \propto \pi(\theta) \int_0^\pi \pi_\eta(\eta).f(C|\theta,\eta) d \eta \propto \pi(\theta) \int_0^\pi f(C | \theta, \eta) d \eta = \pi(\theta) f(C | \theta).
\end{equation}
The main issue is to find function $f(\theta)$ with respect to distribution function $P(\theta)$, that is mean compute $E[f]= \int f(\theta) P(\theta) d\theta$. We assume that the computation of the expected value is hard. The main idea behind the sampling methods, is to obtain samples $\theta^{(\ell)}$ such that be i.i.d and $P(\theta)$ was distribution function, thus we have 
$\hat{f}=\dfrac{1}{L} \sum_{\ell = 1}^L f(\theta^{(\ell)})$. The importance resampling method estimate of distribution function as follows
\begin{itemize}
\item Let select of the sample from $P(\theta)$ is hard and use of simpler function $q(\theta)$ which is known for proposal distribution, we have
\begin{align*}
E[f]=\int f(\theta)P(\theta) d\theta= \int f(\theta) \dfrac{P(\theta)}{q(\theta)}.q(\theta) d \theta \simeq \dfrac{1}{L} \sum_{\ell=1}^L \dfrac{P(\theta^{(\ell)})}{q(\theta^{(\ell)})} f(\theta^{(\ell)});
\end{align*}
\item We select $M$ sample from $\pi(\theta)$ with simple random sampling in the interval $I$ for a specified $k$. We select $k$ sample from the population as a set which use to the computation of distance in equation \ref{eq:elastic_dis};
\item We obtain weight from $w(\theta_i)= f( C | \theta_i)$ and then sampling with size $s$ such that probability of each $\theta_i$ is $\dfrac{w(\theta_i)}{\sum_{i=1}^M w(\theta_i)}$;
\item We choose $k$ that has minimum elastic distance for posterior samples.
\end{itemize}
\subsection*{Finite Mixture Models}
Let $Y=(Y_1, \ldots, Y_n)$ are random variables with size $n$ such that $Y_j$ is a $1$-dimensional random variable  $f(y_j)$ on $\mathbb{R}$ and $y=(y_1, \ldots, y_n)$ denotes an observed random samples. We assume that the density $f(y_j)$ of $Y_j$ can be written in the form
\begin{equation}\label{eq:finiteMix}
f(y_j)= \sum_{i=1}^g \pi_i f_i(y_j)
\end{equation}
where the $f_i(y_j)$ are densities and $\pi_i$ is known as mixing proporation, which is $0 \leq \pi_i \leq 1$, and $\sum_{i=1}^g \pi_i = 1$.
\subsubsection*{Identifiability}
Let the probability space $(\Omega, \mathcal{F}, \mathcal{P})$ that $\Omega$ is sample space, $\mathcal{F}$ be a $\sigma$-algebra of events, $\mathcal{P}= \{ P_\theta, \theta \in \Omega \}$ be a probability measure, and $X$ be a random variable. For $X \sim P_{\theta, \xi}$, there exist pairs $(\theta_1, \xi_1)$ and $(\theta_2, \xi_2)$ with $\theta_1 \neq \theta_2$ for which $P_{\theta_1,\xi_1} = P_{\theta_{2},\xi_{2}}$ showing the parameter $\theta$ to be unidentifiable. It has to be noted that a parameter that is unidentifiable cannot be estimated consistently (see \cite{Lehmann} and \cite{Eskandari2016})
\begin{definition}
The model is identifiable if for any two parameters $\psi$ and $\psi^\ast$ we have 
\begin{align}
\sum_{i=1}^g \pi_i f_i(y_j) = \sum_{i=1}^{g{^\ast}} \pi_{i}^\ast f_i(y_j)
\end{align}
for all possible value of $y_j$, implies $g = g^\ast$ and $\psi = \psi^\ast$.
\end{definition}
\subsection*{Persistence Landscape}
A simplicial complex $K$ is defined for representing a manifold and triangulation of topological space $X$. $K$ is a combinatorial object that is stored easily in computer memory and can be constructed by several methods in high dimensions with any metric space. A subcomplex $L$ of simplicial complex $K$ is a simplicial complex such that $L \subseteq K.$ A filtration of simplicial complex $K$ is a nested sequence of subcomplexs such that $K^0 \subseteq K^1 \subseteq \ldots \subseteq K^m$. To see how this object is created, the readers can refer to (\cite{KhuyenGeneralizedCech}, \cite{ChambersVietorisRips}, \cite{DeyGraphInduced}, and \cite{SilvaWitness}). The simplex tree is a data structure with efficient implementation of the basic operation and topological operation on simplicial complexes of any dimension with a trie structure (\cite{Boissonnat2014}).  \newline
The fundamental group of space $X$ ($\pi_1(X,x_0)$ at the basepoint $x_0$), as an important functor in algebraic topology, consist of loops and deformations of loops. The fundamental group is one of the homotopy group $\pi_n(X,x_0)$ that has a higher differentiating power from space $X$; however, this invariant of topological space $X$  depends on smooth maps and is very complicated to compute in high dimensions. Thus, we must use an invariant of topological space that is computable on the simplicial complex. Homology groups show how cells of dimension $n$ attach to subcomplex of dimension $n-1$ or describe holes in the dimension of $n$ (connected components, loops, trapped volumes,etc.). The nth homology group is defined as $H_n= \ker \partial_n / \mbox{im} \partial_{n+1} = Z_n / B_n$ such that $\partial_n$ is the boundary homomorphism of subcomplexs, $Z_n$ is the cycle group and $B_n$ is boundary group. The nth Betti number $\beta_n$ of a simplicial complex $K$ is defined as $\beta_n = rank(Z_n) - rank(B_n)$. Through filtration step, we tend to extract invariant that remains fixed in this process, thus persistence homology satisfies this criterion for space-time analysis. Let $K^l$ be a filtration of simplicial complex $K$, the pth persistence of nth homology group of $K^l$ is $H_n^{b,d} = Z_n^b / (B_n^{b+d} \cap Z_n^b)$. The Betti number of the pth persistence of the nth homology group is defined as $\beta_n^{b,d}$ for the rank of free subgroup $(H_n^{b,d})$. To visualize persistence in space-time analysis, we should find the interval of $(i,j)$ that is invariant constantly through the filtration and obtain a topological summary from the point cloud (see \cite{AfraBook} and  \cite{Zomorodian2005}). \newline
Now, by rewriting the Betti number of the pth persistence of nth homology group, we have
\begin{equation*}
\lambda (b,d) =
  \begin{cases}
    \beta^{b,d}     &  \mbox{if} \ b \leq d \\
    0   & otherwise
  \end{cases}
\end{equation*}
where $b$ is the birth and $d$ is the death. To convert $\lambda(b,d)$ function to a decreasing function, we change coordinate on it, Let $m = \dfrac{b+d}{2}$ and $h = \dfrac{d-b}{2}$. The rescaled rank function is
\begin{equation*}
\lambda (m,h) =
  \begin{cases}
    \beta^{m-h,m+h}     &  \mbox{if} \ h \geq 0 \\
    0   & otherwise
  \end{cases}
\end{equation*}
\begin{definition}
The persistence landscape is a function  $\lambda: \mathbb{N} \times \mathbb{R} \rightarrow \bar{\mathbb{R}}$ where $\bar{\mathbb{R}}$ denoted the extended real numbers (introduced by \cite{JMLR:v16:bubenik15a}). In the other words, persistence landscape is sequence of function $\lambda_k : \mathbb{R} \rightarrow \bar{\mathbb{R}}$ such that
\begin{equation}\label{eq:landscapeBubenik}
\lambda_k (t) = \sup (m \geq 0 | \beta^{t-m,t+m} \geq k).
\end{equation}
\end{definition}
\subsection*{Nonparametric Density Estimation}
The goal of nonparametric  density estimation is to estimate $f$ with as few assumptions about $f$ as possible. The estimator is defined by $\widehat{f}_n$. We evaluate the quality of an estimator $\widehat{f}_n$ with the risk, or integrated mean squared error, $R=\mathbb{E}(L)$, where
\begin{align}
L = \int (\widehat{f}_n(x) - f(x))^2 dx
\end{align}
is the integrated squared error loss function. The estimators depend on the smoothing parameter, $h$, chosen by minimizing an estimate of the risk. The loss function, refer to as function $h$ from now on, is
\begin{align*}
L  \begin{array}[t]{l} = {\displaystyle \int} (\widehat{f}_n(x) - f(x))^2 dx \\
= {\displaystyle \int} \widehat{f}_n^2(x) dx - 2 {\displaystyle \int} \widehat{f}_n(x) f(x) dx + {\displaystyle \int} f^2(x) dx.\end{array}
\end{align*}
The last term does not depend on $h$ so minimizing the loss is equivalent to minimizing the expected value; therefore, the cross-validation estimator of risk is
\begin{equation} \label{eq:crossValidation}
\widehat{J}(h) = \int \big( \widehat{f}_n(x) \big)^2 dx - \dfrac{2}{n} \sum_{i=1}^n \widehat{f}_{(-i)} (X_i)
\end{equation}
where $\widehat{f}_{(-i)}$ is the density estimator obtained after removing the $i^{th}$ observation. 
\begin{theorem} \label{eq:theorem612}
Suppose that $f^\prime$ is absolutely continuous and that $\int \big( f^\prime (u) \big)^2 du < \infty$, Then,
\begin{align}
R(\widehat{f}_n , f ) = \dfrac{h^2}{12} \int \big( f^\prime (u) \big)^2 du + \dfrac{1}{nh} + o(h^2) + o (\dfrac{1}{n}).
\end{align}
Where $x_n = o(a_n)$, suggesting that $\lim_{n \rightarrow \infty} x_n / a_n = 0$.  The value $h^\ast$ that minimizes (theorem \ref{eq:theorem612}) is
\begin{align}
h^\ast = \dfrac{1}{n^{1/3}} \Bigg( \dfrac{6}{\int ( f^\prime (u) )^2 du } \Bigg)^{1/3}.
\end{align}
With this choice of binwidth,
\begin{align}
R(\widehat{f}_n , f) \sim \dfrac{C}{n^{2/3}}
\end{align}
where $C = (3/4)^{2/3} \Big( \displaystyle\int \big( f^\prime (u) \big)^2 du \Big)^{1/3}$.
\end{theorem}
The integrate mean square error is variance $+\text{bias}^2$, so we have $ \Big(\dfrac{1}{nh} + o (\dfrac{1}{n})\Big)+\Big((f^\prime (\tilde{x_j}))^2 \dfrac{h^3}{12} + O(h^4)\Big)$. The proof of theorem \ref{eq:theorem612} can be seen in \cite{WassermanNonparametric}. We see that with an optimally chosen binwidth, the risk decreses to $0$ at rate of $n^{-2/3}$. Moreover, it can be seen that kernel estimators converge at the faster rate  $n^{-4/5}$ and no faster rate is possible in a certain sense.\\
We discuss kernel density estimators, which are smoother and can converge to the true density faster. Here, the word kernel refers to any smooth function $K$ such that $K(x) \geq 0$ and
\begin{equation}
\int K(x) dx = 1, \ \int xK(x) dx = 0 \ \mbox{and} \ \sigma_K^2 \equiv \int x^2 K(x) dx > 0.
\end{equation}
Some commonly used kernels are the following
\renewcommand{\arraystretch}{2.0}
\begin{table}[ht]
\begin{center}
\begin{tabular}{ || p{5cm} | p{5cm} ||  }
\hline
  the Gaussian kernel: & $K(x) = \dfrac{1}{\sqrt{2 \pi}} \exp^{-x^{2}/2}$  \\
  \hline
  the tricube kernel: & $K(x) = \dfrac{70}{81} \Big( 1- |x|^3 \Big)^3 I(x)$ \\
  \hline
\end{tabular}
\end{center}
\end{table}
where 
\begin{align*}
I(x) = \begin{cases}
1 & \mbox{if} \ |x| \ \leq 1 \\
0 & \mbox{otherwise}
\end{cases}
\end{align*}
\begin{definition} \label{def:densityEstimator}
Given a kernel $K$ and a positive number $h$, called the bandwidth, the kernel density estimator is defined as
\begin{equation}
\widehat{f}_n(x) = \dfrac{1}{n} \sum_{i=1}^n \dfrac{1}{n} K \Big( \dfrac{x - X_i}{h} \Big).
\end{equation}
\end{definition}
\begin{theorem}
Assume that $f$ is continuous at $x$, $h_n \rightarrow 0$ and $n h_n \rightarrow \infty$ as $n \rightarrow \infty$. Then, by the weak low of large number(WLLN), $\widehat{f}_n(x) \rightarrow f(x)$.
\end{theorem}
\begin{proof}
Please see \cite{WassermanNonparametric}.
\end{proof}
\section{Experiment} \label{section:Applications}
The persistence landscapes in this section are calculated using TDA \cite{TDARFasy}, and we use alphashape3d and MASS packages for visualization and density estimation, respectively, in \textsf{R} programming language. Alpha shape is the generalization of the convex hull of landmarks torus such that $0 \leq \alpha \leq \infty$. We use only alphashape3d package for visualization of the population sampling with a uniform distribution on the torus. 
\subsection*{Torus}
Let $R$  be the major radius and $r$ as the minor radius. We use an explicit equation in cartesian coordinates for a torus as figure \ref{fig:torus}, which is
\begin{equation}
\Big( R - \sqrt{x^2 + y^2} \Big)^2 + z^2 = r^2.
\end{equation}
 \begin{figure}[ht]
\centering
\includegraphics[scale=0.25]{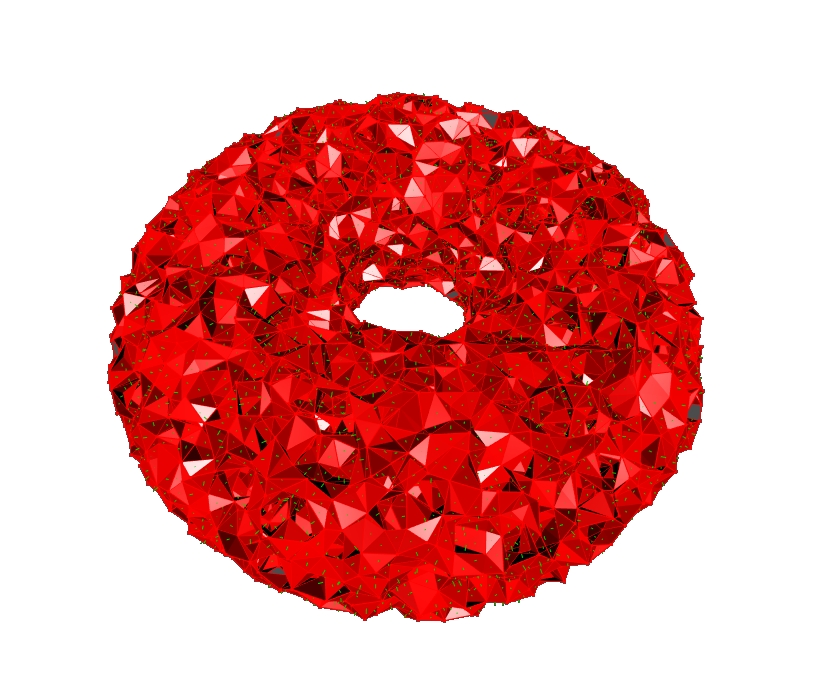}
\caption{The alpha shape \cite{alphaShape} of a finite set with $10000$ landmark from torus with a uniform distribution such that minor radius $r=2$, major radius $R=3$, and $\alpha=0.25$.}
\label{fig:torus}
\end{figure}
To reach the purpose of this study, we took two: \begin {enumerate*}[label=(\roman*)]
\item sampling from population with sampling importance resampling;
\item computation of the finite mixture model from samples.
\end {enumerate*}
First, we set the sampling default to $M=10000$ landmarks with respect to a uniform distribution (figure \ref{fig:torus}) on the parametrized curve as a population Then, we select $s=1000$ landmarks with respect to uniform distribution from $M$ and obtained the probability of each $\theta_i$, $1 \leq i \leq M$ with $k=3$.
\begin{figure}[ht]
\centering
\includegraphics[scale=0.25]{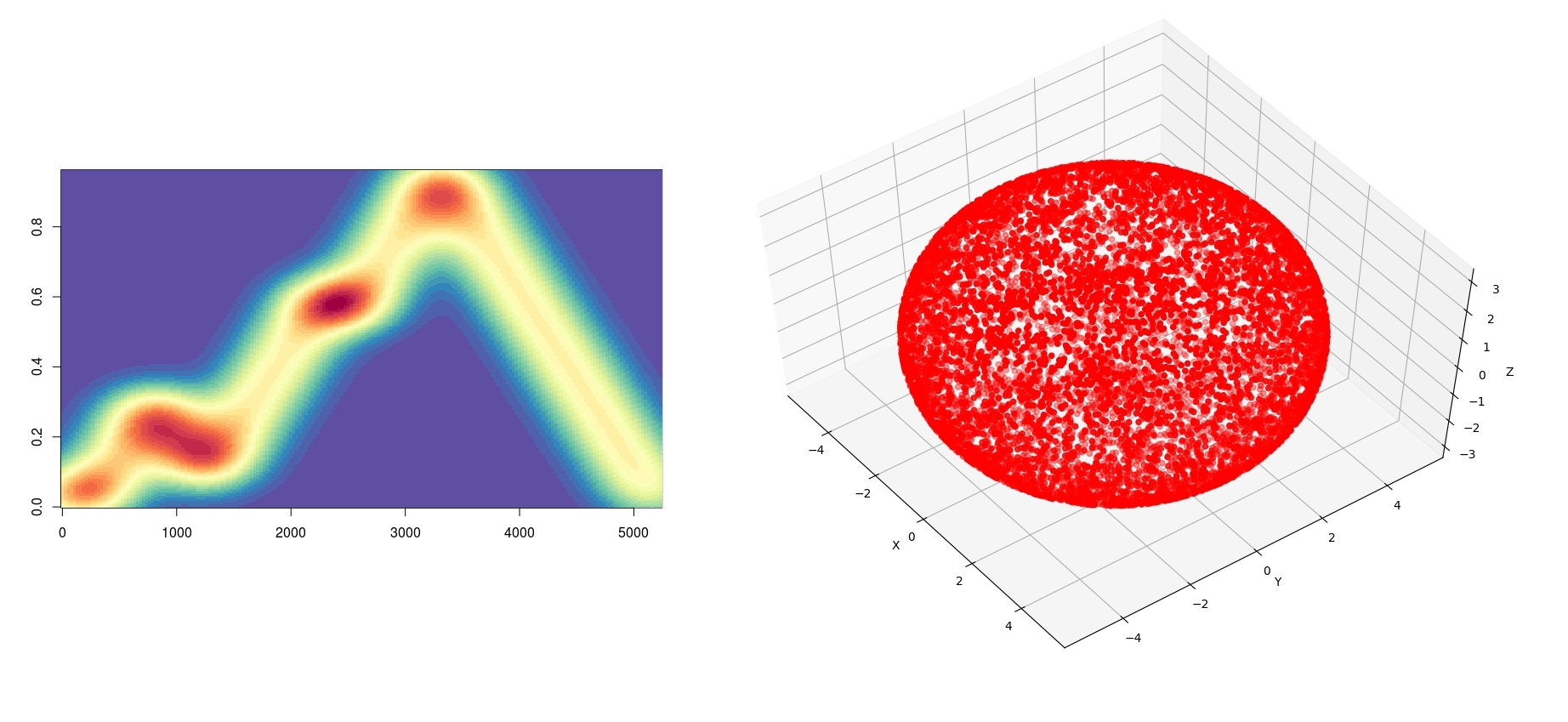} \\
\includegraphics[scale=0.25]{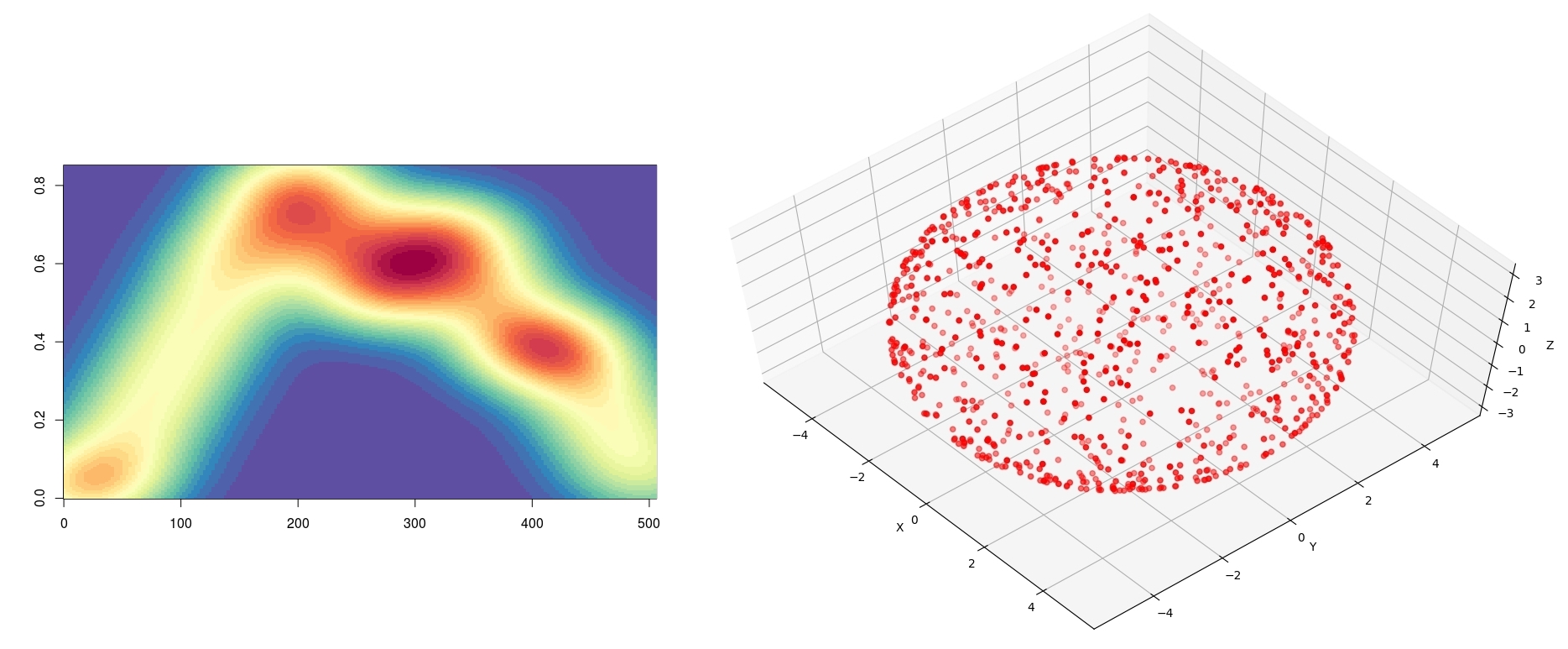} 
\caption{We compute persistence landscape for $10000$ landmarks as population from the torus and visualized the density of these in row $1$ with mean $0.4280793$. Also, We computed the persistence landscape for $1000$ sample with sampling importance resampling approach and visualized the density of these in row $2$ with mean of $0.4264059$ and deviation of $0.07830429$.}
\label{fig:torusLanscape}
\end{figure}
Like some other statistical problems, choosing the number of landmarks $k$ is not a trivial issue. In \cite{Automatic}, for a specified $k$, the posterior sample $\theta_1, \ldots, \theta_s$, the applied $d^2(\beta, L_{\theta_{i}})$ can be computed for $i=1, \ldots ,s$. Finally, the average of distance can be computed for different value of $k$ as shown in figure \ref{fig:sampling} that yielded $k=4$ with $\text{distance}=4.197569$, where $k$ is defined in by equation \ref{eq:elastic_dis}. Table \ref{tbl:selectK} presents the computed distances with different $k$ values.\newline
\begin{table}[ht]
\begin{center}
\begin{tabular}{ *2c }    \toprule
\emph{k} & \emph{elastic distance }  \\\midrule 
$1$ & $4.773316$ \\ 
$2$ & $4.393263$ \\ 
$3$ & $4.318741$ \\
$4$ & $4.197569$ \\
$5$ & $4.583322$ \\
$6$ & $4.369554$ \\ 
$7$ & $4.204580$ \\ 
$8$ & $4.381379$ \\
$9$ & $4.228707$ \\
$10$ & $4.408091$ \\
\bottomrule
 \hline
\end{tabular}
\end{center}
\caption{We run our algorithm with different value of $k$ and computed elastic distance (equation \ref{eq:elastic_dis}) through sampling importance resampling.}
\label{tbl:selectK}
\end{table}
 \begin{figure}[H]
\centering
\includegraphics[scale=0.5]{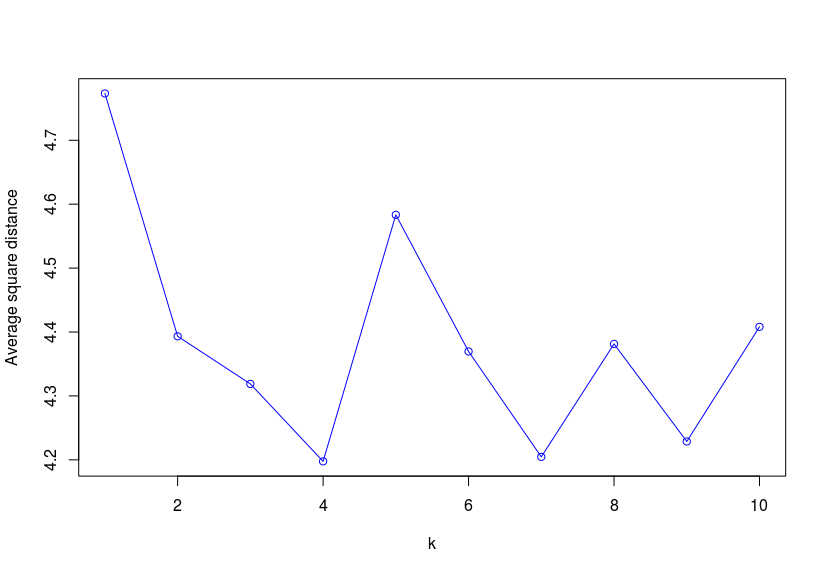}
\caption{We selected $k=4$, which has the minimum elastic distance in sampling importance resampling.}
\label{fig:sampling}
\end{figure}
Afterward, we calculated distances and selected samples with elastic distance, followed by constructing a persistence landcape obtained from the finite mixing of persistence landscapes densities. We construct a filtered simplicial complex as follows. First, we formed the Vietoris-Rips complex $R(X,\epsilon)$, which consists of simplices with vertices in $X = \lbrace x_1,\ldots,x_n \rbrace \subset \mathbb{R}^d$ and diameter at most $\epsilon$. The sequence of Vietoris-Rips complex obtained by gradually increasing the radius $\epsilon$  create a filtration of complexes. We denote the limit of filtration of the Vietoris-Rips complex with $5$ and maximum dimension of homological feature with $1$($0$ for components, $1$ for loops). To compute landscape function in Equation \ref{eq:landscapeBubenik}, we set $t \in [0,5], k = 1$. Figure \ref{fig:torusLanscape} represents persistence landscapes population and sample from the parametrized  curve. According to algorithms in \cite{Pakniat}, we computed confidence interval of the finite mixture model $g=7$ with normal confidence interval $3.472991$ and $3.257315$. We use a Gaussian kernel with bandwidth $h=0.00693$, and we select mixing proportation $\pi_i=\dfrac{1}{g}=0.142857143$ for simplicity. In order to obtain number of components, we calculated IMSE for $100$ times run for each component and plotted mean of run in figure \ref{fig:imse}. So, we selected a minimum value of IMSE with $g=7$ as the final density in equation \ref{eq:finiteMix}. \newline
Eventually, rather than simple sampling from population, we selected only those landmarks that are important on a parametrized curve. Moreover, rather than evaluating one situation of persistence landscape, we assessed different situation of these that direct influence on the precision of IMSE. Therefore, we can obtain a model that is more accurate than before.
\begin{table}[ht]
\begin{center}
\begin{tabular}{ *2c }    \toprule
\emph{g} & \emph{IMSE}  \\\midrule 
$1$ & $0.1283964$ \\ 
$2$ & $0.1312942$ \\ 
$3$ & $0.1443199$ \\
$4$ & $0.1488445$ \\
$5$ & $0.1436486$ \\
\bottomrule
 \hline
\end{tabular}
\quad
\begin{tabular}{ *2c }    \toprule
\emph{g} & \emph{IMSE}  \\\midrule 
$6$ & $0.1275415$ \\
$7$ & $0.1270364$ \\
$8$ & $0.1285071$ \\
$9$ & $0.1359496$ \\
$10$ & $0.1319306$ \\
\bottomrule
 \hline
\end{tabular}
\end{center}
\caption{We run the algorithm $100$ times for each component and present the mean  IMSE for $1 \leq g \leq 10$.}
\label{tbl:bootstrap1}
\end{table}
 \begin{figure}[H]
\centering
\includegraphics[scale=0.5]{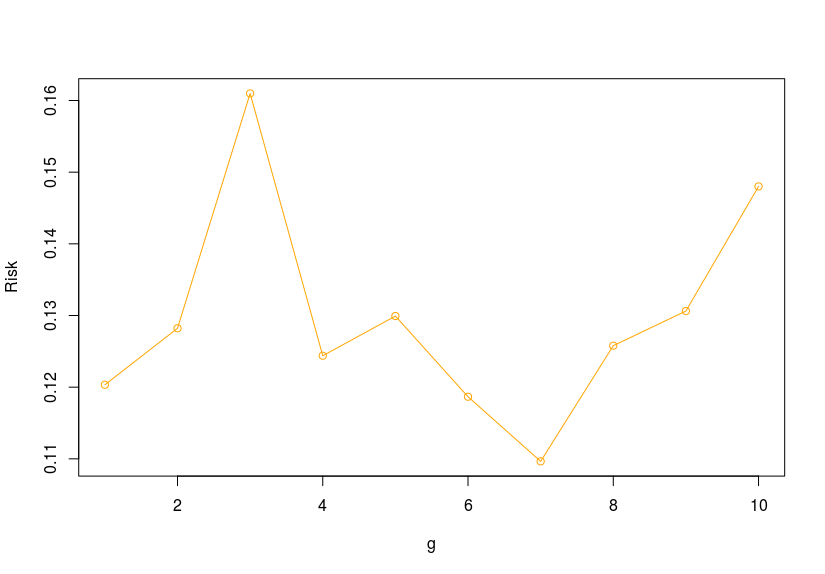}
\caption{IMSE values plotted for $1 \leq g \leq 10$ for the first run.}
\end{figure}
 \begin{figure}[H]
\centering
\includegraphics[scale=0.5]{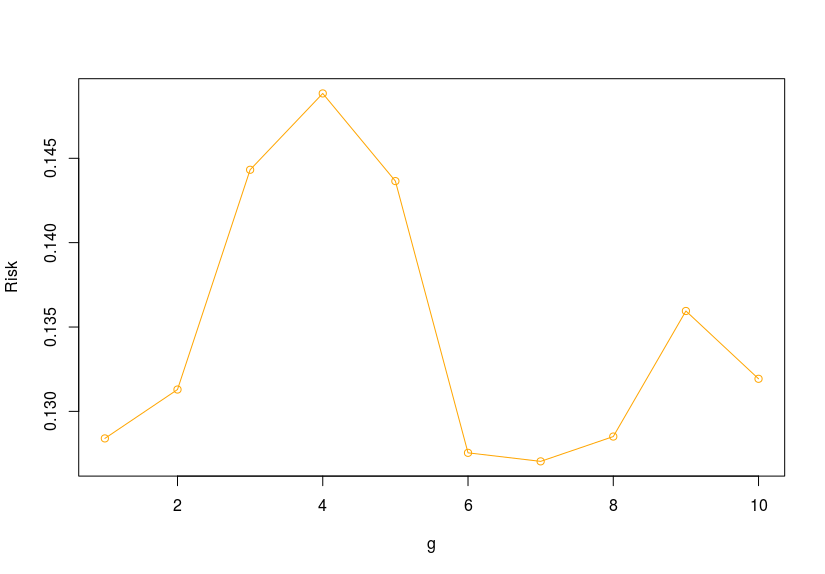}
\caption{The plot for $100$ runs for $1 \leq g \leq 10$ and the obtained mean of IMSEs}
\label{fig:imse}
\end{figure}
\subsection*{Discussion}
The present study has two objectives. The first is how we can sample from an important location on the parametrized curve such that it can reduce space complexity from the dataset. The second is the computation of the finite mixture model of persistence landscapes with mixing proporation according to $\pi_i=\dfrac{1}{g}$.
For the first goal, we represent sampling importance resampling with the elastic metric distance, based on square root velocity function, which is invariant to reparameterization, scaling, and rigid motion. Next, we represent persistence landscapes, which are mixing of persistence landscapes with sampling obtained from the first goal. Thus, it is clear that mixture models give descriptions of entire subgroups rather than assignments of individuals to those subgroups. \newline In figure \ref{fig:torusLanscape}, we computed persistence landscapes from a population with samples slightly different from each other. Although we selected a sample with sampling importance resampling, we obtained a confidence interval of samples leading to the true value of the mean of the population. This difference comes from construct simplicial complex from a sample rather than the population. 
This difference in the value of persistence landscape might be due to fact that the created holes may take different values.
 One of the suggestions for improvement of this difference would be studying the mean of estimator value when choosing $s$ in process of sampling importance resampling. If we also adjust mixing proportion of finite mixture model with respect to densities, we can obtain the $\pi_i$ more accurately. There are some ideas that can be applied to extend this approach. Selection of parameter and shrinkage estimator has been an important approach when dealing with the model complexity.

\end{document}